\documentclass[11pt,letterpaper]{article}
\usepackage{optimalregret}

\newcommand{\cX}{\mathcal X}       
\newcommand{\cY}{\mathcal Y}       
\newcommand{\cH}{\mathcal H}       
\newcommand{\cE}{\mathcal E}       
\newcommand{\cF}{\mathcal F}       
\newcommand{\cV}{\mathcal V}       

\newcommand{\R}{\mathbb R}

\newcommand{\EE}{\mathbb E}
\newcommand{\PP}{\mathbb P}
\newcommand{\ind}{\mathbf 1}         
\newcommand{\Reg}{\operatorname{Reg}}
\newcommand{\fin}{\mathrm{fin}}
\newcommand{\Deltafin}{\Delta_{\fin}}

\newcommand{\LD}{\operatorname{Ldim}}
\newcommand{\SOA}{\operatorname{SOA}}
\DeclareMathOperator{\opt}{opt}
\DeclareMathOperator{\supp}{supp}

\DeclareMathOperator{\argmax}{argmax}

\title{Optimal Multiclass Online Regret\\with Arbitrary Label Spaces}
\author{%
 Sajad Ashkezari\thanks{University of Waterloo, \texttt{sajad.ashkezari@uwaterloo.ca}.}
 \and
 Shai Ben-David\thanks{University of Waterloo, \texttt{shai@uwaterloo.ca}.}
}
\date{}
\hypersetup{pdftitle={Optimal Multiclass Online Regret with Arbitrary Label Spaces},pdfauthor={Sajad Ashkezari and Shai Ben-David}}

\begin{document}
\maketitle

\begin{abstract}
We determine the optimal regret in agnostic multiclass online learning
with arbitrary, possibly infinite, label spaces. For every hypothesis
class of multiclass Littlestone dimension $d<\infty$, we prove a regret
upper bound of $O(\min\{T,\sqrt{dT}\})$ over $T$ rounds, matching the
known lower bound. This resolves a question of \citet{multiclass}
and extends the optimal binary regret bound of \citet{alon}.
Our proof strengthens the one-sided expert cover of \citet{multiclass} to a
fractional cover. We combine this cover with a minimax reduction to
a finite expert game, whose best-response rewards provide the
disagreement bounds needed for chaining.
\end{abstract}

\section{Introduction}
\label{sec:introduction}

We study multiclass agnostic online classification, where a learner
predicts labels from a possibly infinite label space $\cY$ and competes
with the best hypothesis in a given class. The study of multiclass
online learnability was initiated by \citet{daniely}, who introduced
the multiclass Littlestone dimension and showed that it characterizes
learnability in the realizable setting. For a class of Littlestone
dimension $d$, they also proved an agnostic regret upper bound of
$O(\sqrt{dT\log(T|\cY|)})$ over $T$ rounds for finite label spaces,
and a lower bound of $\Omega(\sqrt{dT})$ for $1\le d\le T$.
The dependence of their upper bound on $|\cY|$ left open whether finite
Littlestone dimension suffices for agnostic online learnability when
the label space is infinite.

\citet{multiclass} resolved this question by proving an upper
bound of $O(\sqrt{dT\log(eT/d)})$ for $1\le d\le T$ that does not
depend on the cardinality of the label space. Their result establishes that finite Littlestone
dimension characterizes agnostic online learnability over arbitrary
label spaces. However, a logarithmic gap remained between their upper
bound and the known lower bound. They explicitly asked whether the
optimal regret is $\Theta(\sqrt{dT})$, as established for binary
classification by \citet{alon}. In this work, we answer this question
affirmatively.

\begin{theorem}[Main result, informal]
\label{thm:informal}
The optimal regret in agnostic multiclass online learning with arbitrary,
possibly infinite, label spaces is $\Theta(\sqrt{dT})$ for $T\ge d$.
\end{theorem}

\paragraph{Related work.}
Online learning in the binary setting begins with the characterization of realizable learnability by \citet{littlestone}.
\citet{bps} extended this characterization to agnostic learning, but
their regret upper bound contained an extra logarithmic factor.
Subsequent work connected regret to sequential Rademacher complexity
and developed covering and chaining methods for bounding this
complexity~\cite{rst2010,rst2015,rstmartingale}. Building on this
connection, \citet{alon} removed the logarithmic factor and obtained
the optimal regret bound in the binary setting.

The main ingredient in the proof of \citet{alon} is their notion of \emph{fractional covers}.
For each scale $0<\alpha\le1$, they construct a distribution over
predictors such that, for every hypothesis in the class and every labeled-data sequence, the
predictors that disagree with that hypothesis on at most $\alpha T$
rounds have mass at least $(\alpha/C)^d$, for a universal constant
$C$~\cite[full version, Definition~9.4 and Lemma~9.5]{alon}.
They then approximate each hypothesis at successively finer scales.
Two successive approximations are close to each other because both
are close to that hypothesis. This controls the differences between
successive terms of the chain. Together with the mass bound, summing
these estimates over scales yields the sharp complexity
bound~\cite[full version, Lemma~9.7]{alon}. Related fractional-cover
and majorizing-measure bounds for real-valued classes were developed
by \citet{bdr}.

The difficulty in extending this argument to multiclass learning is
that the available cover provides a different kind of approximation.
The covers of \citet{alon} count every disagreement with the hypothesis being approximated.
When predictions and labels are binary, this is also the number of
rounds on which exactly one of the two predictions is correct.
The expert family of \citet{multiclass} guarantees, for each hypothesis, an expert that
is wrong on at most $d$ rounds on which that hypothesis is correct.
We call these the expert's \emph{deficit rounds}. This guarantee places no restriction on rounds where the expert is correct and
the hypothesis is wrong. These extra correct predictions reduce the
expert's loss, but they can make its correctness sequence very
different from that of the hypothesis. Thus two experts satisfying this
one-sided guarantee need not be close to each other, and the guarantee
alone does not provide the distance bounds needed for chaining.

One might instead apply the binary result to the class of binary losses of the multiclass
hypotheses. However, the Littlestone dimension of this binary
class, called the sequential graph dimension, can be infinite even
when the original multiclass class has dimension
$d=1$~\cite[Example~1]{multiclass}. It is this larger dimension that
determines the sequential Rademacher complexity of the loss
class~\cite[Theorem~7]{multiclass}. A bound on that complexity therefore
cannot in general yield a regret bound in terms of $d$. For finite label spaces, Hanneke et al.~\cite[Theorem~9]{multiclass}
also obtain regret $O(\sqrt{dT\log|\cY|})$ through the sequential
graph dimension, under their minimax assumption. This removes
the extra logarithmic dependence on $T$, but retains a dependence
on the size of the label space.

Our proof starts from the finite expert family of \citet{multiclass} and
strengthens their one-sided guarantee to a fractional cover.
We then study the regret of a learner that predicts with mixtures over this family.
Using minimax duality, as in~\cite{abernethy2009,rst2015}, we express
the value of this game so that the adversary chooses a distribution
on labels before the learner responds. The learner's best response
on each round is then the expert with the largest probability of being correct. The finite family lets us justify this minimax step without topological
assumptions on $\cH$: the learner's weights lie in a compact simplex,
and the one-round game satisfies Sion's theorem~\cite{sion}. The
additional minimax assumption noted in~\cite[Footnote~4]{multiclass}
is therefore unnecessary here.

We use the best-response reward to control the extra correct predictions
of the experts selected in our chain. This gives the disagreement bounds
that the one-sided cover alone does not provide and leads to the optimal
regret bound. Section~\ref{sec:overview} explains the proof in more detail.

Other related work considers general bounded
losses~\cite{unified2025}, transductive learning, where the learner
sees all instances before the first round~\cite{transductive2024},
and learning with multiple correct answers~\cite{multiple2026}.

\section{Model and main result}
\label{sec:model}

Let $\cX$ and $\cY$ be arbitrary nonempty sets, and let
$\varnothing\ne\cH\subseteq\cY^{\cX}$ be a class of total functions.
Fix a positive integer horizon $T$.
Let $[T]=\{1,\ldots,T\}$.
We refer to a sequence of labeled instances
$z_{1:T}=((x_1,y_1),\ldots,(x_T,y_T))\in(\cX\times\cY)^T$
as a \emph{labeled sequence}. Let $z_{<t}$ denote its prefix before round $t$.
For a set $A$, let $\Deltafin(A)$ denote the set of finitely supported probability mass functions on $A$. When $A$ is finite, we may use $\Delta(A)$ instead. All logarithms are natural unless a base is specified.

An online learning algorithm $\mathbb A$ is specified by a sequence of maps
\[
 p_t:(\cX\times\cY)^{t-1}\times\cX\longrightarrow\Deltafin(\cY),
 \qquad t\in[T].
\]
These maps do not take earlier sampled predictions as additional
inputs, matching the convention in~\cite{multiclass}. On round $t$, the adversary presents $x_t$, the learner announces
$p_t=p_t(z_{<t},x_t)$, and the adversary then chooses $y_t$.
The learner incurs loss $1-p_t(y_t)$. We allow the adversary to be adaptive, so its choices may depend on all earlier play. Define the regret on this sequence by
\begin{equation}
 \Reg_T(\mathbb A;z_{1:T})
 =\sum_{t=1}^T(1-p_t(y_t))
 -\min_{h\in\cH}\sum_{t=1}^T\ind\{h(x_t)\ne y_t\}.
 \label{eq:pathwise-regret}
\end{equation}
The minimum exists because the cumulative losses of hypotheses in $\cH$ form a
nonempty subset of $\{0,\ldots,T\}$.
Following Hanneke et al.~\cite{multiclass}, we define the
\emph{regret} of $\mathbb A$ at horizon $T$ by

\begin{equation}
 \begin{aligned}
 \Reg(\mathbb A,T)
 &=\sup_{z_{1:T}\in(\cX\times\cY)^T}\Reg_T(\mathbb A;z_{1:T})\\
 &=\sup_{z_{1:T}\in(\cX\times\cY)^T}
 \left[\sum_{t=1}^T(1-p_t(y_t))
 -\min_{h\in\cH}\sum_{t=1}^T\ind\{h(x_t)\ne y_t\}\right].
 \end{aligned}
 \label{eq:regret-definition}
\end{equation}

We abbreviate the regret on a given sequence by $\Reg_T$ when the
algorithm and sequence are clear. Thus $\Reg(\mathbb A,T)\le R$ is equivalent
to $\Reg_T\le R$ on every labeled sequence.\footnote{One may also sample $\widehat Y_t\sim p_t$ using fresh
randomness after $y_t$ is chosen. Appendix~\ref{app:sampling}
relates expected sampled regret to $\Reg(\mathbb A,T)$,
including against adaptive adversaries.}

Our regret bound is expressed in terms of the multiclass
Littlestone dimension, which we define next.

\begin{definition}[Littlestone dimension~\cite{littlestone,daniely}]
\label{def:littlestone}
A multiclass Littlestone tree of depth $k$ is a complete binary tree
with an instance $x\in\cX$ at each internal node and a label in
$\cY$ on each outgoing edge. The two edges leaving any one node
have distinct labels. A class $V\subseteq\cY^{\cX}$ shatters the tree if, for every root-to-leaf path, some $h\in V$ maps each internal node's instance to the label of the edge taken along that path. The Littlestone dimension $\LD(V)$ is the
supremum of the depths of trees shattered by $V$. We set
$\LD(\varnothing)=-1$ and use $\LD(V)=\infty$ when $V$ shatters trees of arbitrarily large depth.
\end{definition}
The tree has binary branching even when $\cY$ is infinite, and
the two outgoing labels may vary across nodes. Throughout, let
$d=\LD(\cH)<\infty$. The following theorem establishes the optimal regret bound in terms of $d$ and $T$.

\begin{theorem}[Main result: optimal multiclass online regret]
\label{thm:main}
There is a universal constant $C>0$ such that, for every nonempty
$\cH\subseteq\cY^{\cX}$ with $d=\LD(\cH)<\infty$ and every
horizon $T\ge1$, there is an online learning algorithm $\mathbb A$ with finitely supported prediction distributions satisfying
\begin{equation}
 \Reg(\mathbb A,T) \le C\min\{T,\sqrt{dT}\}.
 \label{eq:main-bound}
\end{equation}
\end{theorem}
Together with the lower bounds of Daniely et al.~\cite[Theorems~24 and~26]{daniely}, Theorem~\ref{thm:main} shows that the optimal regret is
$\Theta(\min\{T,\sqrt{dT}\})$.

\begin{remark}[Measurability]
Our regret bound holds on every labeled sequence and requires no measurability assumptions.
\end{remark}

\section{Preliminaries: SOA and finite expert covers}
\label{sec:preliminaries}

\subsection{Version spaces and SOA}
A labeled sequence is realizable by $V\subseteq\cH$ if some
$h\in V$ agrees with all its labels. If, in addition, $\LD(V)=0$,
then $V=\{h\}$. For an instance-label pair, define
\[
 V_{x,y}=\{h\in V:h(x)=y\}.
\]
For nonempty $V$, the Standard Optimal Algorithm~\cite{littlestone,daniely} predicts
\[
 \SOA(V,x)\in\argmax_{y\in\cY}\LD(V_{x,y}).
\]
The maximum exists because the dimensions lie in
$\{-1,0,\ldots,\LD(V)\}$. Fix all tie-breaking choices and a
default label for $\SOA(\varnothing,x)$. When SOA is run
conservatively, it replaces $V$ by $V_{x,y}$ after a mistake with
observed label $y$, and leaves $V$ unchanged after a correct prediction.
The following lemma shows that a mistake from a nonempty version
space reduces its dimension. We include a short proof for completeness.

\begin{lemma}[Rank drop~{\cite{littlestone,daniely}}]
\label{lem:rank-drop}
If $V\subseteq\cH$ is nonempty and $y\ne\SOA(V,x)$, then
$\LD(V_{x,y})\le\LD(V)-1$.
\end{lemma}
\begin{proof}
Let $k=\LD(V)$ and $a=\SOA(V,x)$. If $\LD(V_{x,y})=k$ for
$y\ne a$, maximality implies $\LD(V_{x,a})=k$ as well. Attaching depth-$k$ trees shattered by $V_{x,a}$ and $V_{x,y}$
below a root labeled $x$, with outgoing edges labeled $a$ and $y$,
yields a depth-$(k+1)$ tree shattered by $V$. This contradicts the definition of Littlestone dimension. Therefore, a mistake must reduce the dimension.
\end{proof}
On a sequence realizable by $V$, the version space remains nonempty
and each SOA mistake drops its dimension. Thus conservative SOA
makes at most $\LD(V)$ mistakes~\cite{littlestone,daniely}.

\subsection[A finite family of experts]{A finite family of experts}
\label{subsec:hanneke-cover}
A deterministic expert $v$ is specified by prediction maps
$v_t:(\cX\times\cY)^{t-1}\times\cX\to\cY$.
On round $t$, it predicts $v_t(z_{<t},x_t)$.
For every $I\subseteq[T]$ with $|I|\le\min\{d,T\}$, define
\begin{equation}
 V_t^I=\{g\in\cH:g(x_r)=y_r\text{ for all }r\in I\cap[t-1]\},
 \qquad e_{I,t}=\SOA(V_t^I,x_t).
 \label{eq:soa-experts}
\end{equation}
Let $\cE=(e_I)_{I\subseteq[T],\,|I|\le\min\{d,T\}}$ be the
indexed family of these experts. Distinct update sets are counted separately, even when they define
the same prediction rule. Thus,
\[
 |\cE|=\sum_{k=0}^{\min\{d,T\}}\binom Tk.
\]
This is the expert family used by Hanneke et al.~\cite{multiclass}.

For a fixed labeled sequence and $h\in\cH$, define
\begin{align}
 F_h(v)&=\sum_{t=1}^T\ind\{h(x_t)=y_t,\ v_t\ne y_t\},\nonumber\\
 P_h(v)&=\sum_{t=1}^T\ind\{v_t=y_t,\ h(x_t)\ne y_t\}.
 \label{eq:one-sided-deficits}
\end{align}
We call $F_h(v)$ the \emph{deficit} of $v$ relative to $h$.
It counts the \emph{deficit rounds}, on which $h$ is correct and $v$
is not. The quantity $P_h(v)$ counts the expert's extra correct predictions. The difference $F_h(v)-P_h(v)$ is the expert's excess number
of mistakes relative to $h$. The sum $F_h(v)+P_h(v)$ is the number
of rounds on which their correctness indicators disagree.

The earlier expert constructions~\cite{bps,daniely} guarantee that,
for every $h\in\cH$ and instance sequence, some expert agrees with $h$
on every round. Hanneke et al.~\cite{multiclass}
showed that controlling only the deficit $F_h(v)$ suffices to obtain
agnostic regret bounds independent of the size of the label space.
We state this one-sided guarantee in the following proposition and
include a short proof for completeness.

\begin{proposition}[One-sided guarantee; cf.~{\cite[proof of Theorem~4]{multiclass}}]
\label{prop:hanneke-cover}
For every labeled sequence and every $h\in\cH$, there is an expert
$e_I\in\cE$ such that $F_h(e_I)=|I|\le d$. Consequently,
\begin{equation}
 \sum_{t=1}^T\ind\{e_{I,t}\ne y_t\}
 \le\sum_{t=1}^T\ind\{h(x_t)\ne y_t\}+d.
 \label{eq:hanneke-cover}
\end{equation}
\end{proposition}
\begin{proof}
Fix the sequence and $h$. Starting from $V=\cH$, scan the rounds and
update $V$ exactly when $h(x_t)=y_t$ and $\SOA(V,x_t)\ne y_t$.
Let $I$ be the set of these rounds. Every update preserves $h$ in
the version space and decreases its dimension by
Lemma~\ref{lem:rank-drop}, so $|I|\le d$. The expert $e_I$ maintains exactly these version spaces, and its deficit rounds are precisely the rounds
in $I$. Finally, its excess mistakes equal
$F_h(e_I)-P_h(e_I)\le F_h(e_I)$.
\end{proof}

Hanneke et al.~\cite[Theorem~4]{multiclass} apply exponential
weights to this expert family. For $1\le d\le T$, the standard
expert-advice guarantee~\cite{cb2006} and $|\cE|\le(eT/d)^d$
yield regret at most
\[
 d+\sqrt{\frac{T}{2}\log|\cE|}
 =O\!\left(\sqrt{dT\log(eT/d)}\right).
\]
This bound is independent of $|\cY|$. Our proof removes the
factor $\sqrt{\log(eT/d)}$ by using fractional covers at multiple
scales.

\subsection{New component: one-sided fractional covers}

Alon et al.~\cite{alon} removed the extra logarithmic factor in
the binary regret bound of Ben-David et al.~\cite{bps} by introducing
\emph{fractional covers} and combining them with chaining. A fractional
cover replaces a bound on the number of experts by a lower bound on
the prior mass of experts that approximate each $h\in\cH$. Inspired
by this approach, we extend the one-sided cover of Hanneke et
al.~\cite{multiclass} to a fractional cover in which approximation is
measured by the deficit $F_h(v)$.

\begin{definition}[One-sided fractional cover]
\label{def:fractional-cover}
Fix a finite indexed family $\cE$ of experts, a scale $s\ge0$, and
$\gamma\ge1$. A prior $\nu\in\Delta(\cE)$ is a one-sided
fractional cover of $\cH$ at scale $s$ with size $\gamma$ if
\[
 \nu\{v:F_h(v)\le s\}\ge\gamma^{-1}
 \qquad\text{for every labeled sequence and every }h\in\cH.
\]
\end{definition}
Proposition~\ref{prop:hanneke-cover} provides such a cover at scale $d$
with size $|\cE|$ by taking the uniform prior. For $1\le d\le T$,
Lemma~\ref{lem:cover} constructs a prior on the same expert family at each scale
$s\in[d,T]$, with size at most $(2eT/s)^d$. This bound decreases
as the permitted deficit increases and provides the estimates
needed for chaining.

A symmetric prediction cover controls
$\sum_t\ind\{v_t\ne h(x_t)\}$. In binary classification, where predictions and outcomes share a
two-label alphabet, this is exactly the number of correctness
disagreements. The binary covers used in~\cite{bps,alon} provide this
type of guarantee. Proposition~\ref{prop:hanneke-cover} controls only $F_h(v)$,
leaving $P_h(v)$ unrestricted. In Section~\ref{sec:chaining}, we obtain
the additional control on $P_h(v)$ needed for chaining.

\section{Proof overview}
\label{sec:overview}
We describe the three steps of the proof in the nontrivial regime
$1\le d\le T$. The expert family $\cE$ is always the one defined
in~\eqref{eq:soa-experts}. The cover priors below are tools for the
analysis; they are not the mixture weights played by the learner.

\paragraph{Step 1: construction of one-sided fractional covers.}
Section~\ref{sec:cover} strengthens
Proposition~\ref{prop:hanneke-cover} from existence of one expert to
a lower bound on the mass of suitable experts. For each $s\in[d,T]$
we construct a prior $\nu_s$ on $\cE$ such that, on every sequence,
\[
 \nu_s\{v:F_h(v)\le s\}\ge(s/(2eT))^d
 \qquad\text{for every }h\in\cH.
\]
For simplicity, first assume $s=dm$ for an integer $1\le m\le T/d$. We assign expert $e_I\in\cE$ prior mass proportional to $m^{-|I|}$, with normalizing constant $Z_m\le(eT/(dm))^d$.

Let us briefly describe how we find a set of good experts for a fixed
$h$ and labeled sequence. We inductively construct a tree whose nodes at depth $k$ are tuples $I=(i_1,\ldots,i_k)$ of update times in increasing order. The root is the empty tuple, and we set $i_0=0$. If the expert $e_I$ has fewer than $m$ deficit rounds after
$i_k$, we declare it a leaf. Otherwise, we add $m$ children
$I_r=(i_1,\ldots,i_k,\tau_r)$, where $\tau_r$ is the $r$th deficit
round after $i_k$, for $1\le r\le m$. Each selected update agrees
with $h$ and occurs on an SOA mistake. Thus the version space keeps
$h$ and its dimension drops at every edge, so the tree is finite and
has depth at most $d$.

Each branch simulates SOA, updating its version space only on
selected deficit rounds. Choosing child $r$ skips updates on
the first $r-1$ deficit rounds and updates after
predicting on the $r$th. Since $r\le m$, a branch ending at
depth $k$ corresponds to an expert that has incurred at most
$km$ deficits through round $i_k$. If $k<d$, the leaf condition
allows at most $m-1$ further deficits, giving a total of at most
$km+(m-1)\le dm$. If $k=d$, its version space only contains $h$, so there are no further deficits. Hence every leaf corresponds to an expert with deficit at most $dm$.

Finally, every internal node has exactly $m$ children. Starting at
the root and choosing each child uniformly reaches leaf $I$ with
probability $m^{-|I|}$, and therefore $\sum_{I\text{ a leaf}}m^{-|I|}=1$. Distinct leaves correspond to distinct update sets, so their total
prior mass is $\frac{1}{Z_m}\sum_{I\text{ a leaf}}m^{-|I|} = 1/Z_m\ge(dm/(eT))^d$. For general $s\in[d,T]$, take
$m=\lfloor s/d\rfloor$ and use this prior as $\nu_s$. Since $s/2\le dm\le s$, the experts corresponding to leaves
have deficit at most $s$, and their total prior mass is at least
$(s/(2eT))^d$, as required.

\paragraph[Step 2: the finite-expert dual game.]{Step 2: the finite-expert dual game.}
Section~\ref{sec:minimax} restricts the learner to mixtures over
$\cE$ and derives a dual expression for the value of this game~\cite{abernethy2009,rst2015}. Since $\cE$ is finite, the simplex $\Delta(\cE)$ is compact and convex. Furthermore, the
one-round payoff is continuous and affine in each player's distribution. Sion's minimax theorem~\cite{sion}
therefore applies at each fixed history and instance. Backward
induction then gives the following dual description. At each history,
the adversary chooses an instance $x_t$ and a finitely supported
distribution $q_t$. The learner observes both and chooses weights
$w_t\in\Delta(\cE)$, after which $Y_t\sim q_t$ is drawn.
The learner's conditional correctness reward is
$\sum_{v\in\cE}w_t(v)q_t(v_t)$. Thus, a best response is
to put all weight on an expert maximizing $q_t(v_t)$. Let
\[
 b_t=\max_{v\in\cE}q_t(v_t),\qquad
 W_h=\sum_{t=1}^T\bigl(\ind\{h(x_t)=Y_t\}-b_t\bigr).
\]
We show that the value of the restricted game is
$\sup_{(\mathbf x,\mathbf q)}\EE\max_{h\in\cH} W_h$,
where $(\mathbf x,\mathbf q)$ specifies the adversary's choices at every history.
Restricting the learner to mixtures over $\cE$ can only increase the
minimax value relative to the original game. Thus, to prove our theorem,
it suffices to upper bound the value of this restricted game.

\paragraph{Step 3: bounding the dual value by chaining.} Section~\ref{sec:chaining} fixes an arbitrary dual tree and bounds
$\EE\max_{h\in\cH}W_h$ by chaining. For $v\in\cE$, let
$C_t(v)=\ind\{v_t=Y_t\}$ and define

\[
 W_v=\sum_{t=1}^T(C_t(v)-b_t),\qquad
 M_v=\sum_{t=1}^T(C_t(v)-q_t(v_t)).
\]

Since $q_t(v_t)\le b_t$, we have $W_v\le M_v$. For every
$h\in\cH$ and every realization of the labeled sequence, if
$F_h(v)\le s$, then
\[
 W_h=W_v+F_h(v)-P_h(v)\le W_v+s\le M_v+s.
\]

Combining the cover at scale $s=d$ with the single-expert
concentration bound leaves an extra logarithmic factor. To remove
it, we use chaining, choosing approximating experts at progressively
finer scales and controlling the differences between their
martingale sums. To this end, take scales $s_j=T4^{-j}$ for
$0\le j\le J=\lfloor\log_4(T/d)\rfloor$, so $d\le s_J<4d$.
For a chain $v_0,\ldots,v_J$ with $F_h(v_j)\le s_j$,
we obtain by telescoping
\[
 W_h\le M_{v_0}+\sum_{j=1}^J(M_{v_j}-M_{v_{j-1}})+s_J.
\]
It remains to show that, on a common high-probability event,
every $h\in\cH$ admits such a chain with controlled initial term
and increments. We construct this event using martingale
concentration to control, simultaneously
at all scales, the masses of experts and pairs violating their
concentration bounds under $\nu_{s_j}$ and
$\nu_{s_j}\otimes\nu_{s_{j-1}}$, respectively.

For $G_j(h)=\{v:F_h(v)\le s_j\}$, Step~1 gives
\[
 \nu_{s_j}(G_j(h))\ge(s_j/(2eT))^d.
\]
On the common event, the bounds on the violating masses ensure
that, after conditioning each prior on $G_j(h)$, the failure
probabilities for all single-expert and adjacent-pair bounds sum
to less than one. Independent draws from these conditional priors
therefore satisfy all the bounds with positive probability,
proving that such a chain exists.

The increment bounds depend on the correctness distances
$D(v,v')=\sum_t(C_t(v)-C_t(v'))^2$, which the deficit bounds alone
do not control. For $h$ with $W_h\ge0$, we use
\[
 P_h(v_j)+F_h(v_j)
 =W_{v_j}-W_h+2F_h(v_j)
 \le M_{v_j}+2s_j.
\]
The left-hand side counts correctness disagreements between
$v_j$ and $h$. Thus the bounds on $M_{v_j}$ also control these
disagreements, and consequently the distances between adjacent
experts. Substituting into the increment bounds and summing
yields a high-probability bound on $\max_{h\in\cH}W_h$.
Integrating the resulting tail bound gives
$\EE\max_{h\in\cH}W_h\le C\sqrt{dT}$, uniformly over dual trees.
Step~2 then completes the proof.

\section{One-sided fractional covers for Littlestone classes}
\label{sec:cover}
Assume $1\le d\le T$ and use the expert family from~\eqref{eq:soa-experts}.
Proposition~\ref{prop:hanneke-cover} guarantees an expert with
deficit at most $d$. We now construct one-sided fractional covers using the same expert family.

\begin{lemma}[Weighted coverage]
\label{lem:cover}
For an integer $m$ with $1\le m\le T/d$, set
\begin{equation}
 Z_m=\sum_{k=0}^d\binom Tk m^{-k},\qquad
 \mu_m(e_I)=\frac{m^{-|I|}}{Z_m}.
 \label{eq:prior}
\end{equation}
For every labeled sequence and every $h\in\cH$,
\begin{equation}
 \mu_m\{v:F_h(v)\le dm\}\ge Z_m^{-1},\qquad
 Z_m\le\left(\frac{eT}{dm}\right)^d.
 \label{eq:weighted-cover}
\end{equation}
Consequently, for every real $s\in[d,T]$ there is a prior $\nu_s$ on the same expert family $\cE$, independent of the sequence and of $h$, satisfying
\begin{equation}
 \nu_s\{v:F_h(v)\le s\}\ge\gamma_s^{-1},\qquad
 \gamma_s=(2eT/s)^d.
 \label{eq:real-cover}
\end{equation}
\end{lemma}
\begin{proof}
    Fix a labeled sequence $(x_t,y_t)_{1:T}$ and $h\in \cH$. We will construct a set of good experts with the desired mass under $\mu_m$. We inductively define a rooted tree whose nodes are tuples of update times in increasing order. The root is the empty tuple. A node $I=(i_1, \cdots, i_k)$ with $i_1 < \cdots < i_k$ has version space
    \[
    V_I=\{g\in\cH:g(x_{i_r})=y_{i_r}\text{ for }1\le r\le k\}.
    \]

    Then $V_{\emptyset}=\cH$. Set $i_0=0$. All updates selected below agree with $h$, so $h\in V_I$ at every node. Assume all nodes at depth $k$ have been defined and fix any such node $I$ of length $k$. Consider the deficit rounds after $i_k$ when $\SOA$ uses the fixed version space $V_I$. That is,
    \begin{equation}
    \tau_1<\cdots<\tau_q,
    \qquad
    \{\tau_1,\ldots,\tau_q\}
    =\{t>i_k:h(x_t)=y_t,\ \SOA(V_I,x_t)\ne y_t\}.
     \label{eq:correction-opportunities}
    \end{equation}

    If $q<m$, declare this node a leaf. Otherwise, add exactly $m$ children to this node: $(I,\tau_1),\ (I,\tau_2),\ldots,(I,\tau_m)$. Choosing child $(I,\tau_r)$ skips updates on $r-1\leq m-1$ deficit rounds and updates after predicting on the $r$th.

    At a chosen time $\tau_r$, the defining conditions in~\eqref{eq:correction-opportunities} say that $y_{\tau_r}=h(x_{\tau_r})$ and that SOA predicted a different label. The child's version space therefore contains $h$, and Lemma~\ref{lem:rank-drop} applies:
    \begin{equation}
    h\in V_{(I,\tau_r)},\qquad
    \LD(V_{(I,\tau_r)})\le\LD(V_I)-1.
    \label{eq:correction-rank}
    \end{equation}

    Each update decreases the Littlestone dimension by at least one,
    while the version space remains nonempty. Thus each branch has at most $d$
    updates. At depth $d$, the dimension is zero, so $V_I=\{h\}$
    and SOA incurs no further deficit rounds.

    We now show that each expert $e_I$ corresponding to a leaf $I$ has deficit $F_h(e_I)$ at most $dm$. Take any such leaf $I=(i_1, \cdots, i_k)$. The expert $e_I$ before its first update has version space $V_\emptyset$. It then updates on $i_1$ and continues with $V_{(i_1)}$ and so on. By the above arguments, it has deficit at most $m$ on rounds $i_{r-1} < t \leq i_r$. After its last update, if $k<d$, it incurs at most $m-1$ further deficits by the leaf condition. If $k=d$, it incurs no further deficit since $V_I=\{h\}$ and $\SOA(V_I,x)=h(x)$. Therefore,
    \begin{equation}
    F_h(e_I)\le
    \begin{cases}
    km+(m-1)=(k+1)m-1\le dm-1,&k<d,\\
    dm,&k=d.
    \end{cases}
    \label{eq:leaf-deficit}
    \end{equation}

    This shows that every leaf corresponds to a good expert. It remains to show that their total prior mass under $\mu_m$ satisfies the desired bound. Starting at the root, choose a child uniformly at each internal node until reaching a leaf. Each leaf $I$ then has probability $m^{-|I|}$ of being reached. Since the process will eventually reach a leaf, we have $\sum_{I\text{ a leaf}}m^{-|I|}=1$. Distinct leaves correspond to distinct update sets, since each set has a unique increasing ordering. Hence,
    \begin{equation}
    \begin{split}
    \mu_m\{v:F_h(v)\le dm\}
    &\ge\sum_{I\text{ a leaf}}\mu_m(e_I)\\
    &=\frac{1}{Z_m}\sum_{I\text{ a leaf}}m^{-|I|}
    =\frac{1}{Z_m}.
    \end{split}
    \label{eq:leaf-mass}
    \end{equation} 
    This proves the coverage part of~\eqref{eq:weighted-cover}. To bound $Z_m$, set $a=dm/T\in(0,1]$. Since $a^d\le a^k$ for
    $0\le k\le d$, we have
    \[
    a^d Z_m
    \le \sum_{k=0}^d\binom Tk(a/m)^k
    \le (1+d/T)^T
    \le e^d.
    \]
    Dividing by $a^d$ gives $Z_m\le(eT/(dm))^d$,
    completing~\eqref{eq:weighted-cover}. Finally, fix $s\in[d,T]$ and let $m=\lfloor s/d\rfloor$.
    Then $1\le m\le T/d$ and $s/2\le dm\le s$. Set $\nu_s=\mu_m$.
    The preceding bounds yield
    \[
    \nu_s\{v:F_h(v)\le s\}
    \ge Z_m^{-1}
    \ge \left(\frac{dm}{eT}\right)^d
    \ge \left(\frac{s}{2eT}\right)^d
    =\gamma_s^{-1}.
    \]
    This prior is independent of the labeled sequence and of $h$, proving
    \eqref{eq:real-cover}.
\end{proof}

\section[A finite-expert minimax reduction]{A finite-expert minimax reduction}
\label{sec:minimax}
We adapt the stochastic minimax duality argument of
\citet[Theorem~1]{abernethy2009} and
\citet[Theorem~1 and Section~5.1]{rst2015} to mixtures over
a finite family of experts. We include the proof to verify
the minimax and attainment claims for arbitrary label spaces.

Fix any nonempty finite indexed family $\cE$ of deterministic experts. At a fixed history and current instance,
let $v_t=v_t(z_{<t},x_t)$ denote the prediction of expert $v$.

A learner chooses weights $w\in\Delta(\cE)$, inducing the following
label distribution with support size at most $|\cE|$:
\[
 p_t^w(y)=\sum_{v\in\cE}w(v)\ind\{v_t=y\},\qquad y\in\cY.
\]

We express regret using correctness rewards. For a complete labeled sequence $z_{1:T}$, define
\[
 \opt_{\cH}(z_{1:T})=\max_{h\in\cH}\sum_{t=1}^T\ind\{h(x_t)=y_t\}.
\]
For the learner $\mathbb A$ induced by a mixture rule $\boldsymbol w$,
the regret in~\eqref{eq:pathwise-regret} is then
\begin{equation}
 \Reg_T(\mathbb A;z_{1:T})
 =\opt_{\cH}(z_{1:T})-\sum_{t=1}^T p_t^{\boldsymbol w}(y_t).
 \label{eq:regret-reward}
\end{equation}
Let $\cV_T(\cE,\cH)$ be the minimax value of this regret when the learner
must use such mixtures. Explicitly,
\begin{equation}
 \cV_T(\cE,\cH)
 =\inf_{\boldsymbol w}\sup_{z_{1:T}\in(\cX\times\cY)^T}
 \left[\opt_{\cH}(z_{1:T})-\sum_{t=1}^T p_t^{\boldsymbol w}(y_t)\right],
 \label{eq:restricted-value}
\end{equation}
where $\boldsymbol w=(w_t)_{t=1}^T$ ranges over deterministic mixture
rules $w_t:(\cX\times\cY)^{t-1}\times\cX\to\Delta(\cE)$, and
$p_t^{\boldsymbol w}=p_t^{w_t(z_{<t},x_t)}$ on the given sequence.

To state the duality lemma, we first define the adversary's stochastic trees.
A stochastic tree $(\mathbf{x},\mathbf{q})$ specifies deterministic maps
$\mathbf{x}_t:\cY^{t-1}\to\cX$ and
$\mathbf{q}_t:\cY^{t-1}\to\Deltafin(\cY)$ for $t\in[T]$.
At a node with label history $Y_{<t}$, let
\[
 x_t=\mathbf{x}_t(Y_{<t}),\qquad q_t=\mathbf{q}_t(Y_{<t}).
\]
The next label $Y_t$ is drawn from $q_t$, and the children of the node
correspond to the labels in $\supp(q_t)$. Along the tree, let
$z_{<t}=((x_s,Y_s))_{s<t}$ and $v_t=v_t(z_{<t},x_t)$.
Every tree has finitely many reachable nodes up to the horizon $T$.
Expectations below are taken under its induced finite probability law.

On round $t$, define $b_t=\max_{v\in\cE}q_t(v_t)$ as the largest
expected correctness reward achievable by a mixture of the experts' current predictions against the adversary's distribution $q_t$.

\begin{lemma}[Finite-expert duality]
\label{lem:minimax}
For every nonempty finite indexed family $\cE$ of deterministic experts,
\begin{equation}
 \cV_T(\cE,\cH)
 =\sup_{(\mathbf{x},\mathbf{q})}\EE\max_{h\in\cH}
   \sum_{t=1}^T\left(\ind\{h(x_t)=Y_t\}-b_t\right),
 \label{eq:dual}
\end{equation}
where the supremum is over the stochastic trees defined above. Moreover, there exists a deterministic mixture rule
whose regret is at most $\cV_T(\cE,\cH)$ on every labeled sequence.
\end{lemma}
\begin{proof}
    We first prove a one-round identity and then apply backward induction
    to obtain~\eqref{eq:dual}.
    Fix a labeled history $z_{<t}$ and a current instance $x_t$, so that
    the predictions $v_t$ are fixed. For any bounded function $G:\cY\to\R$, with $G_y=G(y)$, we claim
    the following one-round identity:
    \begin{equation}
    \inf_{w\in\Delta(\cE)}\sup_{y\in\cY}
    \bigl(G_y-p_t^w(y)\bigr)
    =\sup_{q\in\Delta_{\fin}(\cY)}
    \left(\EE_qG_Y-\max_{v\in\cE}q(v_t)\right).
    \label{eq:one-round}
    \end{equation}
    First, observe that
    \[\inf_{w\in\Delta(\cE)}\sup_{y\in\cY}
    \bigl(G_y-p_t^w(y)\bigr) = \inf_{w\in\Delta(\cE)}\sup_{q\in\Delta_{\fin}(\cY)}\EE_q
    \bigl[G_Y-p_t^w(Y)\bigr],
    \]
    since an average cannot exceed the supremum over labels, while point
    masses can approach that supremum arbitrarily closely. By Sion's minimax
    theorem~\cite[p.~174]{sion},
    \[
    \inf_{w\in\Delta(\cE)}\sup_{q\in\Delta_{\fin}(\cY)}\EE_q
    \bigl[G_Y-p_t^w(Y)\bigr] = \sup_{q\in\Delta_{\fin}(\cY)}\inf_{w\in\Delta(\cE)}\EE_q
    \bigl[G_Y-p_t^w(Y)\bigr].
    \]
    Indeed, $\Delta(\cE)$ is compact and convex, and
    $\Delta_{\fin}(\cY)$ is a convex subset of $\ell_1(\cY)$.
    Equip the latter set with the inherited $\ell_1$ topology. Since $G$
    is bounded, the payoff is continuous and affine in each argument,
    so the theorem applies. For fixed $q\in\Delta_{\fin}(\cY)$,
    linearity of expectation gives
    \[
    \inf_{w\in\Delta(\cE)}\EE_q
    \bigl[G_Y-p_t^w(Y)\bigr] = \EE_q G_Y - \sup_{w\in\Delta(\cE)}\EE_q \sum_{v\in \cE}w(v)\ind\{v_t=Y\},
    \]

    and $\EE_q \sum_{v\in\cE}w(v)\ind\{v_t=Y\}
    =\sum_{v\in\cE}w(v)q(v_t)$. Since this is a convex combination of
    the finitely many values $q(v_t)$, its supremum over $w\in\Delta(\cE)$
    equals $\max_{v\in\cE}q(v_t)$, attained by putting all weight on a
    maximizing expert. Therefore,
    \[
    \inf_{w\in\Delta(\cE)}\EE_q[G_Y-p_t^w(Y)]
    =\EE_qG_Y-\max_{v\in\cE}q(v_t).
    \]
    Substituting this equality into the preceding minimax identity
    proves~\eqref{eq:one-round}.

    We now apply this identity recursively. Set $U_T(z_{1:T})=\opt_{\cH}(z_{1:T})$
    and, for $t=T,\ldots,1$, define
    \begin{equation}
    U_{t-1}(z_{<t})
    =\sup_{x_t\in\cX}\inf_{w\in\Delta(\cE)}\sup_{y_t\in\cY}
    \left[U_t(z_{<t},(x_t,y_t))-p_t^w(y_t)\right],
    \label{eq:backward-value}
    \end{equation}
    where $p_t^w$ uses the expert predictions at $(z_{<t},x_t)$.
    The order of the operations follows the order of play. Since
    $0\le \opt_{\cH}\le T$ and $0\le p_t^w(y)\le1$, backward induction gives
    $-(T-t)\le U_t\le T$, so every continuation value is bounded.

    Applying~\eqref{eq:one-round} with
    $G_y=U_t(z_{<t},(x_t,y))$ gives the recursion
    \begin{equation}
    U_{t-1}(z_{<t})
    =\sup_{x_t\in\cX}\sup_{q_t\in\Delta_{\fin}(\cY)}
    \left[\EE_{q_t}U_t(z_{<t},(x_t,Y_t))
    -\max_{v\in\cE}q_t(v_t)\right].
    \label{eq:dual-recursion}
    \end{equation}
    For any fixed stochastic tree $(\mathbf{x},\mathbf{q})$, its choices
    at each reached node therefore satisfy
    \[
    \EE[U_t(z_{1:t})\mid Y_{<t}]-b_t\le U_{t-1}(z_{<t}).
    \]
    Taking expectations and summing over $t$ gives
    \[
    \EE\left[\opt_{\cH}(z_{1:T})-\sum_{t=1}^T b_t\right]\le U_0.
    \]

    Conversely, fix $\varepsilon>0$. At every reached node, choose
    $x_t,q_t$ within $\varepsilon/T$ of the supremum
    in~\eqref{eq:dual-recursion}. Each $q_t$ has finite support, so these
    choices define a finite tree of depth $T$. At each of its nodes,
    \[
    \EE[U_t(z_{1:t})\mid Y_{<t}]-b_t
    \ge U_{t-1}(z_{<t})-\varepsilon/T.
    \]
    Taking expectations and summing over the $T$ rounds, we obtain an
    expected payoff of at least $U_0-T(\varepsilon/T)=U_0-\varepsilon$.
    Letting $\varepsilon$ tend to zero, we obtain
    \[
    U_0=\sup_{(\mathbf{x},\mathbf{q})}
    \EE\left[\opt_{\cH}(z_{1:T})-\sum_{t=1}^T b_t\right].
    \]

    We next show that $\cV_T(\cE,\cH)\ge U_0$.
    Fix any deterministic mixture rule $\boldsymbol w$ and any stochastic tree.
    Since $w_t(z_{<t},x_t)$ and the expert predictions are determined before $Y_t$
    is drawn,
    \[
    \EE[p_t^{\boldsymbol w}(Y_t)\mid Y_{<t}]
    =\sum_{v\in\cE}w_t(z_{<t},x_t)(v)q_t(v_t)\le b_t.
    \]
    Hence, taking expectations and summing over the rounds,
    \begin{equation*}
    \EE\Reg_T
    =\EE\!\left[\opt_{\cH}(z_{1:T})-\sum_{t=1}^T p_t^{\boldsymbol w}(Y_t)\right] \geq\EE\!\left[\opt_{\cH}(z_{1:T})-\sum_{t=1}^T b_t\right].
    \end{equation*}
    Therefore,
    \[
    \begin{aligned}
    \cV_T(\cE,\cH)
    &=\inf_{\boldsymbol w}\sup_{z_{1:T}\in(\cX\times\cY)^T}
    \left[\opt_{\cH}(z_{1:T})-\sum_{t=1}^T p_t^{\boldsymbol w}(y_t)\right]\\
    &\geq\inf_{\boldsymbol w}\sup_{(\mathbf{x},\mathbf{q})}
    \EE\!\left[\opt_{\cH}(z_{1:T})-\sum_{t=1}^T p_t^{\boldsymbol w}(Y_t)\right]\\
    &\geq\sup_{(\mathbf{x},\mathbf{q})}
    \EE\!\left[\opt_{\cH}(z_{1:T})-\sum_{t=1}^T b_t\right] = U_0.
    \end{aligned}
    \]
    
    For the reverse inequality, fix a history $z_{<t}$ and an instance $x_t$. The inner supremum
    in~\eqref{eq:backward-value} is continuous in $w$: changing $w$ to
    $w'$ changes every term by at most $\|w-w'\|_1$, since
    $|p_t^w(y)-p_t^{w'}(y)|\le\|w-w'\|_1$ uniformly over $y$.
    Compactness of $\Delta(\cE)$ therefore ensures that a minimizing mixture exists. Fix one such mixture $w_t^*(z_{<t},x_t)$ at every history and instance,
    and let $\boldsymbol w^*=(w_t^*)_{t=1}^T$. For every labeled sequence, we have
    \[
    \begin{aligned}
    U_t(z_{1:t})-p_t^{\boldsymbol w^*}(y_t)
    &{}\le \sup_{y\in\cY}
    \left[U_t(z_{<t},(x_t,y))-p_t^{\boldsymbol w^*}(y)\right]\\
    &{}=\inf_{w\in\Delta(\cE)}\sup_{y\in\cY}
    \left[U_t(z_{<t},(x_t,y))-p_t^w(y)\right]\\
    &\le U_{t-1}(z_{<t}).
    \end{aligned}
    \]
    The first inequality takes the supremum over labels, the equality
    uses the minimizing choice of $w_t^*$, and the last inequality
    takes the supremum over instances in~\eqref{eq:backward-value}.
    Summing over $t$ and using $U_T=\opt_{\cH}$ gives
    \[
    \Reg_T
    =\opt_{\cH}(z_{1:T})-\sum_{t=1}^T p_t^{\boldsymbol w^*}(y_t)\le U_0.
    \]
    Thus $\cV_T(\cE,\cH)\le U_0$, with a deterministic mixture
    rule achieving this bound on every sequence.

    Together with $\cV_T(\cE,\cH)\ge U_0$, this proves~\eqref{eq:dual}. The rule $\boldsymbol w^*$ constructed above attains this value.
\end{proof}

\section{Bounding the dual game by localized chaining}
\label{sec:chaining}
The last step is to bound the dual value using the one-sided fractional
covers. We first select suitable chains on a common event, then use
the benchmark $b_t$ to control correctness disagreements and sum the
centered increments.

\subsection{Concentration and simultaneous chain selection}
\label{subsec:selection}
We return to the expert family $\cE$ of Lemma~\ref{lem:cover} and assume
$1\le d\le T$. In this subsection, we establish concentration estimates
and a simultaneous chain-selection result that will be used to bound
the dual game value in Lemma~\ref{lem:minimax}, uniformly over
stochastic trees.
Fix one such tree $(\mathbf{x},\mathbf{q})$. Throughout this section,
$\PP$ and $\EE$ refer to its label draws, and
$\cF_t=\sigma(Y_1,\ldots,Y_t)$, with $\cF_0$ trivial. The current
instance $x_t$, distribution $q_t$, and each expert's prediction $v_t$
are $\cF_{t-1}$-measurable, and the conditional law of $Y_t$ is $q_t$.

For $v,v'\in\cE$, define
\begin{align}
 C_t(v)&=\ind\{v_t=Y_t\},&
 M_v&=\sum_{t=1}^T\bigl(C_t(v)-q_t(v_t)\bigr),\nonumber\\
 W_v&=\sum_{t=1}^T\bigl(C_t(v)-b_t\bigr),&
 D(v,v')&=\sum_{t=1}^T\bigl(C_t(v)-C_t(v')\bigr)^2.
 \label{eq:processes}
\end{align}
Thus $M_v$ is the centered cumulative correctness of $v$, while
$D(v,v')$ counts rounds on which exactly one of the two experts is
correct. Since
$\EE[C_t(v)\mid\cF_{t-1}]=q_t(v_t)$, the partial sums defining $M_v$
form a martingale. Moreover, $q_t(v_t)\le b_t$ gives
\begin{equation}
 W_v=M_v-\sum_{t=1}^T\bigl(b_t-q_t(v_t)\bigr)\le M_v.
 \label{eq:domination}
\end{equation}
For $h\in\cH$, let $C_t(h)=\ind\{h(x_t)=Y_t\}$ and
$W_h=\sum_{t=1}^T(C_t(h)-b_t)$.
To bound $\EE\max_{h\in\cH}W_h$, we will use chains of experts
with small deficits and controlled centered increments.
The experts in these chains may depend on the realized sequence,
so concentration bounds for fixed experts cannot be applied
directly to them.

Following the fractional-cover method of~\cite{alon,bdr}, we first
control, with high probability, the prior masses of experts and
pairs that violate the concentration bounds. For each $h\in\cH$,
we then condition each prior on experts with small deficit relative
to $h$. Using the cover guarantees, we show that a chain drawn from
these conditional priors has positive probability of satisfying
all the concentration bounds. Such a chain therefore exists for
every $h\in\cH$ on the same event, without a union bound over $\cH$.

\begin{lemma}[Concentration for fixed experts]
\label{lem:concentration}
For every fixed $v\in\cE$ and every $a>0$,
\begin{equation}
 \PP\{M_v>\sqrt{Ta/2}\}\le e^{-a}.
 \label{eq:hoeffding}
\end{equation}
For every fixed ordered pair $v,v'\in\cE$, every $a>0$, and every
$0<\lambda\le1/2$,
\begin{equation}
 \PP\{M_v-M_{v'}>\lambda D(v,v')+a/\lambda\}\le e^{-a}.
 \label{eq:empirical}
\end{equation}
\end{lemma}
Lemma~\ref{lem:concentration} follows from the conditional Hoeffding
bound and an exponential-supermartingale argument. Its proof is given
in Appendix~\ref{app:concentration}.

\paragraph{Mass of experts violating the concentration bounds.}
We now consider fractional covers at different scales. On a single
high-probability event, we will control the prior mass of experts
violating the concentration bounds at each scale, and the product-prior
mass of violating expert pairs at adjacent scales.

Set scales $s_j=T4^{-j}$ for $0 \leq j \leq J=\left\lfloor\log_4(T/d)\right\rfloor$. These scales decrease from $s_0=T$ to $d\le s_J<4d$.
For each $j$, take the prior $\mu_j=\nu_{s_j}$ given in
Lemma~\ref{lem:cover}, and define
\[
 \gamma_j=(2eT/s_j)^d=(2e4^j)^d,\qquad
 G_j(h)=\{v\in\cE:F_h(v)\le s_j\}.
\]
By~\eqref{eq:real-cover},
\begin{equation}
 \mu_j(G_j(h))\ge\gamma_j^{-1}
 \qquad\text{for every labeled sequence and every }h\in\cH.
 \label{eq:scale-coverage}
\end{equation}

Fix a tail parameter $\tau\ge0$ and let $\eta_j=2^{-j-4}$. Define
\begin{align}
 a_j&=\tau+\log\gamma_j+2\log(1/\eta_j),
 & B_j&=\sqrt{Ta_j/2},&&0\le j\le J,\nonumber\\
 a'_j&=\tau+\log(\gamma_j\gamma_{j-1})+2\log(1/\eta_j),
 &&&&1\le j\le J.
 \label{eq:chain-thresholds}
\end{align}
We also fix numbers $\lambda_j\in(0,1/2]$ for $1\le j\le J$,
whose values will be specified later. We will bound the following
random variables under the tree's label law:
\begin{align*}
 A_j&=\mu_j\{v\in\cE:M_v>B_j\},&&0\le j\le J,\\
 Q_j&=(\mu_j\otimes\mu_{j-1})
       \{(v,v')\in\cE^2:
       M_v-M_{v'}>\lambda_jD(v,v')+a'_j/\lambda_j\},&&1\le j\le J.
\end{align*}

\begin{lemma}[Simultaneous control of violating masses]
\label{lem:good-masses}
For every fixed $\tau\ge0$ and every fixed choice of
$\lambda_1,\ldots,\lambda_J\in(0,1/2]$, the event
\begin{equation}
 \Omega_\tau
 =\bigcap_{j=0}^J\left\{A_j\le\frac{\eta_j}{\gamma_j}\right\}
 \cap\bigcap_{j=1}^J
 \left\{Q_j\le\frac{\eta_j}{\gamma_j\gamma_{j-1}}\right\}
 \label{eq:good-masses}
\end{equation}
satisfies
\begin{equation}
 \PP(\Omega_\tau)\ge1-\frac{e^{-\tau}}4.
 \label{eq:good-probability}
\end{equation}
\end{lemma}
\begin{proof}
    By definition, $A_j=\sum_{v\in\cE}\mu_j(v)\ind\{M_v>B_j\}$. Using Lemma~\ref{lem:concentration}, we obtain
    \begin{align*}
    \EE A_j
    &=\sum_{v\in\cE}\mu_j(v)\PP\{M_v>B_j\}\\
    &\le\sum_{v\in\cE}\mu_j(v)e^{-a_j}
    =e^{-a_j}
    =e^{-\tau}\eta_j^2/\gamma_j.
\end{align*}
Similarly, $\EE Q_j\le e^{-\tau}\eta_j^2/(\gamma_j\gamma_{j-1})$.
By Markov's inequality,
\[
 \PP\{A_j>\eta_j/\gamma_j\}
 \le\frac{\EE A_j}{\eta_j/\gamma_j}
 \le e^{-\tau}\eta_j,
\]
and

\[
 \PP\{Q_j>\eta_j/(\gamma_j\gamma_{j-1})\}
 \le e^{-\tau}\eta_j.
\]

A union bound completes the proof, since $\sum_{j=0}^J\eta_j + \sum_{j=1}^J\eta_j < \frac{1}{4}$.
\end{proof}

\paragraph{Good experts at all scales.} We now show that on $\Omega_\tau$, for each $h$ there exists a chain of ``good'' experts at all scales.

\begin{lemma}[Simultaneous chain selection]
\label{lem:selection}
On the event $\Omega_\tau$ from Lemma~\ref{lem:good-masses}, every
$h\in\cH$ admits a tuple
$v_0,\ldots,v_J\in\cE$ satisfying
\begin{align}
 F_h(v_j)&\le s_j,\qquad M_{v_j}\le B_j,
 &&0\le j\le J,\label{eq:selected-single}\\
 M_{v_j}-M_{v_{j-1}}
 &\le\lambda_jD(v_j,v_{j-1})+a'_j/\lambda_j,
 &&1\le j\le J.\label{eq:selected-pair}
\end{align}
\end{lemma}
\begin{proof}
    Fix a realized sequence $z_{1:T}\in\Omega_\tau$ and
    $h\in\cH$. Define the conditional priors
    \[
    \mu_j^h(v)=
    \frac{\mu_j(v)\ind\{v\in G_j(h)\}}{\mu_j(G_j(h))},
    \qquad v\in\cE.
    \]
    These are well-defined since
    $\mu_j(G_j(h))\ge\gamma_j^{-1}>0$.
    We use the probabilistic method: it suffices to show that the tuples
    satisfying the lemma's conditions have positive probability under
    $\bigotimes_{j=0}^J\mu_j^h$. Let $\PP_{\mathrm{aux}}$ denote this
    auxiliary probability. Let $V_0,\ldots,V_J$ be independent draws
    with $V_j\sim\mu_j^h$. For every $0\le j\le J$,
    \begin{align*}
    \PP_{\mathrm{aux}}\{M_{V_j}>B_j\}
    &=\frac{\mu_j(\{v:M_v>B_j\}\cap G_j(h))}{\mu_j(G_j(h))}\\
    &\le\gamma_j\mu_j\{v:M_v>B_j\}
    =\gamma_j A_j\le\eta_j,
    \end{align*}
    where the last inequality follows from $z_{1:T}\in\Omega_\tau$.

    Similarly, independence of $V_j$ and $V_{j-1}$ gives, for every
    $1\le j\le J$,
    \begin{align*}
    &\PP_{\mathrm{aux}}\!\left\{
    M_{V_j}-M_{V_{j-1}}
    >\lambda_jD(V_j,V_{j-1})+a'_j/\lambda_j\right\}\\
    &\qquad\le
    \frac{Q_j}{\mu_j(G_j(h))\mu_{j-1}(G_{j-1}(h))}\\
    &\qquad\le\gamma_j\gamma_{j-1}Q_j\le\eta_j.
    \end{align*}
    The first inequality bounds the mass of violating pairs in
    $G_j(h)\times G_{j-1}(h)$ by their total product-prior mass $Q_j$.
    The remaining inequalities use~\eqref{eq:scale-coverage} and
    $z_{1:T}\in\Omega_\tau$.

    Since $V_j\in G_j(h)$ almost surely, the deficit bounds
    $F_h(V_j)\le s_j$ hold automatically. By a union bound, the
    auxiliary probability that any of the single-expert or adjacent-pair
    concentration bounds fails is at most
    \[
    \sum_{j=0}^J\eta_j+\sum_{j=1}^J\eta_j<\frac14.
    \]
    Thus the tuples satisfying all the lemma's conditions have positive
    auxiliary probability, so at least one such tuple exists. The sequence
    $z_{1:T}\in\Omega_\tau$ and $h\in\cH$ were arbitrary,
    proving the simultaneous statement.
\end{proof}

\subsection{Controlling disagreements and completing the bound}
\label{subsec:dual-value}
We now use the selected chains to bound $\EE\max_{h\in\cH}W_h$
uniformly over stochastic trees. Fix $\tau\ge0$ and set, for
$1\le j\le J$,
\[
 K_j=B_j+B_{j-1}+2s_j+2s_{j-1},\qquad
 \lambda_j=\min\left\{\frac12,\sqrt{\frac{a'_j}{K_j}}\right\}.
\]
These parameters depend only on $T,d,j,\tau$ and are fixed. Let $\Omega_\tau$ be the event in
Lemma~\ref{lem:good-masses} for this choice of parameters.
The remaining proof consists of the following steps.

\paragraph{Controlling correctness disagreements.}
On $\Omega_\tau$, fix $h\in\cH$ with $W_h\ge0$ and take the chain
provided by Lemma~\ref{lem:selection}. We show that
\[
 \sum_{t=1}^T\bigl(C_t(v_j)-C_t(h)\bigr)^2\le B_j+2s_j,
 \qquad 0\le j\le J.
\]
Indeed, the definitions in~\eqref{eq:one-sided-deficits} yield
\[
 \sum_{t=1}^T\bigl(C_t(v_j)-C_t(h)\bigr)^2=P_h(v_j)+F_h(v_j).
\]
Moreover,
\[
 P_h(v_j)-F_h(v_j)
 =\sum_{t=1}^T\bigl(C_t(v_j)-C_t(h)\bigr)
 =W_{v_j}-W_h
 \le M_{v_j}\le B_j,
\]
where we used $W_h\ge0$, $W_{v_j}\le M_{v_j}$, and
Lemma~\ref{lem:selection}. Consequently,
\[
 \sum_{t=1}^T\bigl(C_t(v_j)-C_t(h)\bigr)^2
 =\bigl(P_h(v_j)-F_h(v_j)\bigr)+2F_h(v_j)
 \le B_j+2s_j.
\]

Since the correctness indicators take values in $\{0,1\}$, whenever
$C_t(v_j)\ne C_t(v_{j-1})$, at least one of them differs from $C_t(h)$.
Thus, for $1\le j\le J$,
\begin{align*}
 D(v_j,v_{j-1})
 &\le \sum_{t=1}^T\bigl(C_t(v_j)-C_t(h)\bigr)^2
     +\sum_{t=1}^T\bigl(C_t(v_{j-1})-C_t(h)\bigr)^2\\
 &\le B_j+B_{j-1}+2s_j+2s_{j-1}
 =K_j.
\end{align*}

\paragraph{Telescoping and summing.}
By~\eqref{eq:selected-pair} and the distance bound just proved,
\[
 M_{v_j}-M_{v_{j-1}}
 \le\lambda_jK_j+\frac{a'_j}{\lambda_j}.
\]
If $a'_j\le K_j/4$, then $\lambda_j=\sqrt{a'_j/K_j}$ and the
right-hand side equals $2\sqrt{K_ja'_j}$. Otherwise,
$\lambda_j=1/2$ and it is $K_j/2+2a'_j\le4a'_j$. Thus, in either
case,
\[
 M_{v_j}-M_{v_{j-1}}\le2\sqrt{K_ja'_j}+4a'_j.
\]
Moreover,
\[
 W_h-W_{v_J}=F_h(v_J)-P_h(v_J)\le F_h(v_J)\le s_J.
\]
Using $W_{v_J}\le M_{v_J}$ and telescoping along the chain gives
\begin{align}
 W_h
 &\le M_{v_J}+s_J\nonumber\\
 &=M_{v_0}+\sum_{j=1}^J(M_{v_j}-M_{v_{j-1}})+s_J\nonumber\\
 &\le B_0+s_J+
 \sum_{j=1}^J\left(2\sqrt{K_ja'_j}+4a'_j\right).
 \label{eq:telescoped}
\end{align}
We derived this bound for $W_h\ge0$. Since the right-hand side is
nonnegative, it also bounds $W_h$ when $W_h<0$. Therefore it bounds
$\max_{h\in\cH}W_h$ on $\Omega_\tau$.

We now estimate the sum in~\eqref{eq:telescoped}. The calculation in
Appendix~\ref{app:chain-sum} shows that its right-hand side is at most
$C(1+\tau)\sqrt{dT}$ for a universal constant $C$. Consequently,
\begin{equation}
 \max_{h\in\cH}W_h\le C(1+\tau)\sqrt{dT}
 \qquad\text{on }\Omega_\tau.
 \label{eq:dual-pathwise-bound}
\end{equation}

\paragraph{From the tail bound to regret.}
By~\eqref{eq:dual-pathwise-bound}, the event
$\{\max_{h\in\cH}W_h>C(1+\tau)\sqrt{dT}\}$ is contained in
$\Omega_\tau^c$. Lemma~\ref{lem:good-masses} therefore gives
\begin{equation}
 \PP\left\{\max_{h\in\cH}W_h>C(1+\tau)\sqrt{dT}\right\}
 \le\frac{e^{-\tau}}4,\qquad\tau\ge0.
 \label{eq:dual-tail}
\end{equation}
Let $U=\max_{h\in\cH}W_h$ and $R=C\sqrt{dT}$.
Since $U$ can be negative, we integrate the tail of its positive part
$U_+=\max\{U,0\}$. Splitting the integral at $R$ gives
\begin{align*}
 \EE U
 &\le\EE U_+
 =\int_0^\infty\PP\{U>u\}\,du\\
 &\le R+\int_R^\infty\PP\{U>u\}\,du\\
 &=R+R\int_0^\infty\PP\{U>R(1+\tau)\}\,d\tau\\
 &\le R+\frac R4\int_0^\infty e^{-\tau}\,d\tau
 =\frac54C\sqrt{dT}.
\end{align*}
Here the second inequality bounds the probability by one on $[0,R]$,
the next line substitutes $u=R(1+\tau)$, and the last line
uses~\eqref{eq:dual-tail}.

The bound holds uniformly over stochastic trees. Taking the supremum over trees and applying
Lemma~\ref{lem:minimax}, we conclude that
\[
 \cV_T(\cE,\cH)
 =\sup_{(\mathbf{x},\mathbf{q})}
 \EE\max_{h\in\cH}W_h
 \le\frac54C\sqrt{dT}.
\]
The same lemma provides a deterministic mixture rule whose regret is at most this value on every labeled sequence. Taking the supremum over labeled sequences yields the same bound
on $\Reg(\mathbb A,T)$.
This proves Theorem~\ref{thm:main} when $1\le d\le T$.

If $d=0$, $\cH$ consists of a single function. Predicting this function gives zero regret on every
sequence. If $d>T$, any strategy has regret at most $T$, since
its cumulative loss is at most $T$ and the minimum cumulative loss over $\cH$ is
nonnegative. These cases establish the stated
$O(\min\{T,\sqrt{dT}\})$ bound for all finite $d$.

\section*{AI disclosure}
The main result of this paper was obtained using OpenAI's GPT-6
Astra model. In a separate, unpublished project, we had introduced
the notion of one-sided covers, inspired by \citet{multiclass}.
While developing that project, we studied \citet{multiclass} and
\citet{alon} extensively with the model. We then asked it to
develop a fractional version of our one-sided cover, following
\citet{alon}, and use it to remove the extra logarithmic factor
from the regret upper bound. The model produced the proof that
forms the basis of this paper. We subsequently checked,
summarized, and rewrote the proof. We also used AI to assist
with writing and editing the manuscript. The authors take
full responsibility for the paper's content and correctness.

\clearpage
\bibliographystyle{optimalalpha}
\bibliography{references}

\appendix
\section{Expected sampled regret}
\label{app:sampling}

For sampled predictions $\widehat Y_t$, define the sampled regret by
\begin{equation}
 \widehat{\Reg}_T
 =\sum_{t=1}^T\ind\{\widehat Y_t\ne y_t\}
 -\min_{h\in\cH}\sum_{t=1}^T\ind\{h(x_t)\ne y_t\}.
 \label{eq:sampled-regret}
\end{equation}
The adversary may use $p_t$ and earlier sampled predictions to
choose $y_t$, but cannot observe the current draw $\widehat Y_t$.

\begin{lemma}
\label{lem:sampling}
Let $\mathbb A$ be an online learning algorithm interacting with an
adaptive adversary, possibly randomized. Sample each prediction
$\widehat Y_t$ from $p_t$ using fresh randomness after $y_t$ is chosen.
Assume that all losses and the minimum cumulative loss over $\cH$ are measurable.
If $\mathcal G_t$ denotes the information available before this draw, then
\begin{equation}
 \EE[\ind\{\widehat Y_t\ne y_t\}\mid\mathcal G_t]
 =1-p_t(y_t)\quad\text{almost surely}.
 \label{eq:conditional-sampling}
\end{equation}
Then
\begin{equation}
 \EE\widehat{\Reg}_T=\EE\Reg_T\le\Reg(\mathbb A,T).
 \label{eq:sampling-regret}
\end{equation}
\end{lemma}
\begin{proof}
Let $L_T^*=\min_{h\in\cH}\sum_{t=1}^T\ind\{h(x_t)\ne y_t\}$.
The tower property gives
\begin{align*}
 \EE\widehat{\Reg}_T
 &=\EE\!\left[\sum_{t=1}^T
   \EE[\ind\{\widehat Y_t\ne y_t\}\mid\mathcal G_t]-L_T^*\right]\\
 &=\EE\!\left[\sum_{t=1}^T(1-p_t(y_t))-L_T^*\right]
 =\EE\Reg_T\le\Reg(\mathbb A,T),
\end{align*}
where the inequality follows from the definition of $\Reg(\mathbb A,T)$.
\end{proof}

Taking the supremum over adaptive adversaries whose interactions
satisfy the lemma's assumptions gives
\[
 \sup_{\text{adaptive adversaries}}\EE\widehat{\Reg}_T
 =\Reg(\mathbb A,T).
\]
The upper bound follows from the lemma. For the reverse inequality, every fixed
labeled sequence can be played by an adversary, and on that sequence
$\EE\widehat{\Reg}_T=\Reg_T(\mathbb A;z_{1:T})$.
Thus Theorem~\ref{thm:main} also bounds worst-case expected sampled regret.

\section{Proof of the concentration lemma}
\label{app:concentration}

\begin{proof}[Proof of Lemma~\ref{lem:concentration}]
Fix $v\in\cE$. Conditional on $\cF_{t-1}$, the variable $C_t(v)$
is Bernoulli with mean $q_t(v_t)$. The conditional Hoeffding bound
therefore gives, for every fixed $\theta>0$,
\[
 \EE\!\left[
 e^{\theta(C_t(v)-q_t(v_t))}\mid\cF_{t-1}\right]
 \le e^{\theta^2/8}.
\]
Iterating conditional expectations over the $T$ rounds yields
$\EE e^{\theta M_v}\le e^{T\theta^2/8}$. Hence, by Markov's inequality,
for every $u>0$,
\[
 \PP\{M_v>u\}
 \le\exp\{-\theta u+T\theta^2/8\}.
\]
Taking $\theta=4u/T$ gives $\PP\{M_v>u\}\le e^{-2u^2/T}$.
Setting $u=\sqrt{Ta/2}$ proves~\eqref{eq:hoeffding}.

For~\eqref{eq:empirical}, fix $v,v'\in\cE$ and
$0<\lambda\le1/2$, and let
\[
 X_t=C_t(v)-C_t(v'),\qquad
 m_t=\EE[X_t\mid\cF_{t-1}]=q_t(v_t)-q_t(v'_t).
\]
We use the elementary inequality
\begin{equation}
 e^{u-u^2}\le1+u,\qquad |u|\le1/2.
 \label{eq:scalar-exponential}
\end{equation}
Indeed, $f(u)=\log(1+u)-u+u^2$ satisfies $f(0)=0$ and
$f'(u)=u(1+2u)/(1+u)$, so its minimum on this interval is zero.
Since $|X_t|\le1$, applying~\eqref{eq:scalar-exponential} with
$u=\lambda X_t$ gives
\begin{align*}
 &\EE\!\left[
 e^{\lambda(X_t-m_t)-\lambda^2X_t^2}\mid\cF_{t-1}\right]\\
 &\qquad=e^{-\lambda m_t}
 \EE\!\left[e^{\lambda X_t-\lambda^2X_t^2}
 \mid\cF_{t-1}\right]\\
 &\qquad\le e^{-\lambda m_t}(1+\lambda m_t)\le1.
\end{align*}
Consequently,
\[
 Z_t=\exp\!\left\{
 \lambda\sum_{s=1}^t(X_s-m_s)-\lambda^2\sum_{s=1}^tX_s^2
 \right\},\qquad Z_0=1,
\]
is a nonnegative supermartingale. In particular,
\[
 \EE Z_T
 =\EE\exp\!\left\{\lambda(M_v-M_{v'})-\lambda^2D(v,v')\right\}
 \le1.
\]
Another application of Markov's inequality gives
\[
 \PP\{M_v-M_{v'}>\lambda D(v,v')+a/\lambda\}
 =\PP\{Z_T>e^a\}\le e^{-a},
\]
which proves~\eqref{eq:empirical}. The realized distance $D(v,v')$
appears inside the exponential throughout the argument, while
$\lambda$ remains fixed.
\end{proof}

\section{Estimating the chain sum}
\label{app:chain-sum}
Assume $1\le d\le T$ and $\tau\ge0$, and retain the notation of
Subsection~\ref{subsec:dual-value}. Define
\[
 S_\tau=B_0+s_J+
 \sum_{j=1}^J\left(2\sqrt{K_ja'_j}+4a'_j\right).
\]
We show that $S_\tau\le C(1+\tau)\sqrt{dT}$ for a universal constant
$C$.

Substituting $\gamma_j=(2e4^j)^d$ and
$\eta_j=2^{-j-4}$ into~\eqref{eq:chain-thresholds} gives
\begin{align*}
 a_j
 &=\tau+d[1+(2j+1)\log2]+(2j+8)\log2
 \le8(d+\tau)(j+1),\\
 a'_j
 &=\tau+d[2+4j\log2]+(2j+8)\log2
 \le8(d+\tau)(j+1),
\end{align*}
where the second line applies to $j\ge1$. Both bounds use $d\ge1$ and
$\log2<3/4$. Consequently,
\[
 B_j\le2\sqrt{T(d+\tau)(j+1)}.
\]
Since $s_j=s_{j-1}/4$, for $j\ge1$ we obtain
\begin{align*}
 K_j
 &=B_j+B_{j-1}+2s_j+2s_{j-1}\\
 &\le4\sqrt{T(d+\tau)(j+1)}+\frac52T4^{-(j-1)}.
\end{align*}
Multiplying this bound by $a'_j\le8(d+\tau)(j+1)$ and using
$\sqrt{a+b}\le\sqrt a+\sqrt b$ for $a,b\ge0$ yields
\begin{align*}
 2\sqrt{K_ja'_j}
 &\le8\sqrt2\,T^{1/4}(d+\tau)^{3/4}(j+1)^{3/4}\\
 &\quad+4\sqrt5\,\sqrt{T(d+\tau)}\,
 2^{-(j-1)}\sqrt{j+1}.
\end{align*}
The second term decreases geometrically with the scale. The first
comes from the individual bounds $B_j$ used to control correctness
disagreements and has the smaller power $T^{1/4}$.

The sums appearing in~\eqref{eq:telescoped} satisfy
\begin{align*}
 \sum_{j=1}^J2^{-(j-1)}\sqrt{j+1}
 &\le\sum_{j=1}^{\infty}2^{-(j-1)}(j+1)=6,\\
 \sum_{j=1}^J(j+1)^{3/4}&\le(J+1)^{7/4},\\
 \sum_{j=1}^J(j+1)&=\frac{J(J+3)}2\le(J+1)^2.
\end{align*}
Together with $B_0\le2\sqrt{T(d+\tau)}$ and $s_J<4d$, these bounds
yield
\begin{align}
 S_\tau
 &\le56\Bigl[\sqrt{T(d+\tau)}
 +T^{1/4}(d+\tau)^{3/4}(J+1)^{7/4}\nonumber\\
 &\hspace{22mm}+(d+\tau)(J+1)^2+d\Bigr]
.
 \label{eq:intermediate-rate}
\end{align}

It remains to check that the powers of $J+1$ do not introduce an
extra logarithmic factor. Let $r=T/d\ge1$. Then
$J+1\le1+\log r$ and $d+\tau\le d(1+\tau)$, so
\begin{align*}
 \frac{T^{1/4}(d+\tau)^{3/4}(J+1)^{7/4}}{\sqrt{dT}}
 &\le(1+\tau)^{3/4}r^{-1/4}(1+\log r)^{7/4},\\
 \frac{(d+\tau)(J+1)^2}{\sqrt{dT}}
 &\le(1+\tau)r^{-1/2}(1+\log r)^2.
\end{align*}
Both factors involving $r$ are uniformly bounded for $r\ge1$:
they are continuous and tend to zero as $r\to\infty$.
The other two terms satisfy
$\sqrt{T(d+\tau)}\le\sqrt{dT}\sqrt{1+\tau}$ and
$d\le\sqrt{dT}$. Since both $\sqrt{1+\tau}$ and
$(1+\tau)^{3/4}$ are at most $1+\tau$, it follows
from~\eqref{eq:intermediate-rate} that, for a universal constant $C$,
\begin{equation}
 S_\tau\le C(1+\tau)\sqrt{dT}.
 \label{eq:chain-sum}
\end{equation}

\end{document}